\documentclass[accepted]{uai2026} 
                        
\usepackage[american]{babel}

\usepackage{amsmath}
\usepackage{amssymb}
\usepackage{amsfonts}
\usepackage{booktabs}
\usepackage{amsthm}
\usepackage{thmtools}
\usepackage{thm-restate}
\usepackage{enumitem}
\usepackage{multirow}

\usepackage{amsthm,amsmath,amssymb}
\usepackage{thmtools}

\declaretheoremstyle[
  headfont=\bfseries,
  bodyfont=\normalfont,
  spaceabove=4pt,
  spacebelow=4pt,
  headpunct=.,
  postheadspace=0.5em
]{tight}

\newcommand{\Ghat}{\widehat{G}}
\newcommand{\fg}{\mathfrak{g}}

\usepackage{tabularx}
\usepackage{natbib} 
\usepackage{mathtools} 
\usepackage{booktabs} 
\usepackage{tikz} 
\usetikzlibrary{arrows.meta,positioning,calc,fit,backgrounds,shapes.geometric}
\usepackage{pgfplots}
\pgfplotsset{compat=1.18}
\usepackage[table]{xcolor}

\usepackage[most]{tcolorbox}
\tcolorboxenvironment{summarybox}{
  colback=gray!10,
  colframe=gray!40,
  boxrule=0.5pt,
  arc=2pt,
  left=6pt,right=6pt,top=6pt,bottom=6pt
}

\definecolor{headerblue}{HTML}{8090B8}       
\definecolor{augerino}{HTML}{E8F4EC} 
\definecolor{augerino_tab}{HTML}{F2FAF4  } 

\definecolor{ourrow}{HTML}{D6E4F7} 
\definecolor{ourrow_tab}{HTML}{EEF4FC } 
\definecolor{rowA}{HTML}{F2F6FC}             
\definecolor{rowB}{HTML}{FFFFFF}             
\definecolor{headertext}{HTML}{FFFFFF}       
\definecolor{accentgreen}{HTML}{1A6B3C}      
\definecolor{rulecolor}{HTML}{B0C4DE}
\definecolor{rowalt}{HTML}{F7F7F7}
\definecolor{highlightblue}{rgb}{0.9, 0.95, 1.0}
\definecolor{highlightgray}{rgb}{0.95, 0.95, 0.95}

\title{Spectral Discovery of Continuous Symmetries via Generalized Fourier Transforms}

\author[1]{Pavan ~Karjol}
\author[1]{Kumar ~Shubham}
\author[1]{Prathosh ~AP}
\affil[1]{%
    \centering
    Department of Electrical Communication Engineering \\
    Indian Institute of Science \\
    Bengaluru, Karnataka
}

\usepackage{amsmath,amsfonts,bm}

\def\eqref#1{equation~\ref{#1}}

\def\1{\bm{1}}

\def\vf{{\bm{f}}}

\DeclareMathAlphabet{\mathsfit}{\encodingdefault}{\sfdefault}{m}{sl}
\SetMathAlphabet{\mathsfit}{bold}{\encodingdefault}{\sfdefault}{bx}{n}

\def\gD{{\mathcal{D}}}

\def\gH{{\mathcal{H}}}

\def\gX{{\mathcal{X}}}

\def\sC{{\mathbb{C}}}

\def\sR{{\mathbb{R}}}

\def\sT{{\mathbb{T}}}

\def\sZ{{\mathbb{Z}}}

\newcommand{\R}{\mathbb{R}}

\begin{document}
\maketitle







\begin{abstract}

Continuous symmetries are fundamental to many scientific and learning problems, yet they are often unknown a priori. 
Existing symmetry discovery approaches typically search directly in the space of transformation generators or rely on learned augmentation schemes. 
We propose a fundamentally different perspective based on spectral structure.

We introduce a framework for discovering continuous one-parameter subgroups using the Generalized Fourier Transform (GFT). 
Our central observation is that invariance to a subgroup induces structured sparsity in the spectral decomposition of a function across irreducible representations. 
Instead of optimizing over generators, we detect symmetries by identifying this induced sparsity pattern in the spectral domain.

We develop symmetry detection procedures on maximal tori, where the GFT reduces to multi-dimensional Fourier analysis through their irreducible representations. 
Across structured tasks, including the double pendulum and top quark tagging, we demonstrate that spectral sparsity reliably reveals  one-parameter symmetries. 
These results position spectral analysis as a principled and interpretable alternative to generator-based symmetry discovery.

\end{abstract}

\section{Introduction}

Symmetry is a fundamental structural principle in machine learning. When a target function is invariant or equivariant under a group of transformations, incorporating this structure into the model can substantially improve generalization, reduce sample complexity, and enhance interpretability. 
Translation equivariance in convolutional neural networks~\cite{Cohen2016,kondor2018generalization,cohen2019general} 
and permutation invariance in Deep Sets~\cite{Zaheer2017} 
are canonical examples in which explicit knowledge of the underlying symmetry group 
enables the design of parameter-efficient models with improved generalization performance. However, these approaches rely on explicit knowledge of the underlying symmetry group, an assumption that often does not hold in practical settings.

In many scientific and real-world problems, the relevant symmetry is not known a priori. Such situations commonly arise in physical simulations, molecular modeling, and dynamical systems, where symmetry is implicit rather than directly observable. Automatically identifying the active symmetry from data, therefore, remains a central challenge. To make this challenge more tractable, one can focus on structured families of symmetries that admit analytic characterization.

Among continuous symmetries, one-parameter subgroups occupy a particularly important role. 
They capture structured rotational or continuous transformation behaviors while remaining analytically tractable, and higher-dimensional subgroups can often be understood as compositions of these fundamental components. 
Discovering such one-parameter subgroups provides a natural entry point for understanding more complex transformation structure.

In this work, we propose a spectral framework for discovering continuous symmetries based on the Generalized Fourier Transform (GFT). 
Our central observation is that invariance to a one-parameter subgroup of $SO(n)$ induces structured sparsity in the spectral decomposition of a function restricted to that subgroup. 
Rather than searching directly in the space of generators, we analyze how a learned function decomposes across irreducible representations in the frequency domain. 
Spectral concentration patterns then serve as signatures of invariance.

To build intuition, consider a function on the plane that is invariant to all rotations.
In angular Fourier coordinates, this full rotational invariance forces all nonzero angular frequencies to vanish, leaving spectral mass only at the zero mode.
More generally, invariance to a one-parameter subgroup imposes a weaker constraint: rather than collapsing entirely to the zero frequency, the spectrum is restricted to a structured subset of admissible frequencies determined by the subgroup. Our approach detects symmetry by identifying this induced spectral sparsity pattern.

We develop symmetry detection procedures on maximal tori, where the GFT reduces to a multi-dimensional Fourier transform with directly interpretable frequency structure. 
This setting allows us \textcolor{black}{to translate representation-theoretic constraints into concrete spectral patterns that can be detected computationally.
}
\paragraph{Why spectral rather than generator search?}

Recent approaches to symmetry discovery typically parameterize candidate generators and optimize over transformation actions in the input space. 
Methods such as Augerino \cite{benton2020learning} learn augmentation distributions that promote invariance, while LieGAN \cite{yang2023liegan} learns continuous transformation groups through adversarial objectives. 
These approaches explicitly model transformations and enforce invariance via data-level perturbations or learned group actions.

In contrast, our approach detects symmetry through spectral structure. 
Instead of optimizing over candidate generators, we examine the frequency content of a learned function and identify invariance via structured sparsity across irreducible representations. 
\textcolor{black}{This yields an interpretable and representation-theoretic criterion for symmetry discovery that avoids direct optimization over transformation parameters.}

To our knowledge, prior symmetry discovery methods have not explicitly leveraged representation-level sparsity patterns in the Generalized Fourier Transform domain as a primary mechanism for identifying continuous subgroups.

We integrate this spectral reasoning into the learning pipeline in two complementary ways:
(i) a model architecture that incorporates maximal-torus Fourier features, and
(ii) a regularization mechanism based on the resonance constraint (Corollary~\ref{cor:torus_resonance}), which encourages functions consistent with subgroup invariance by penalizing off-resonant Fourier modes and thereby promoting sparsity in the spectral coefficients.
Empirically, we demonstrate that spectral sparsity reliably reveals one-parameter subgroups and that incorporating this structure improves learning performance across multiple tasks.

\subsection{Contributions}

Our main contributions are as follows:
\begin{itemize}

\item \textbf{A spectral framework for continuous symmetry discovery.}  
We propose a Generalized Fourier Transform (GFT)-based methodology for identifying continuous one-parameter subgroups through structured spectral sparsity. 
On maximal tori, where the GFT reduces to multi-dimensional Fourier analysis, invariance manifests as interpretable concentration patterns across irreducible components.

\item \textbf{Spectral architectures for symmetry-aware learning.}  
We incorporate spectral reasoning directly into the model design through architectures built on maximal-torus Fourier features.  
These architectures expose interpretable frequency structure and enable symmetry discovery to emerge jointly with predictive learning.

\item \textbf{Resonance-based regularization for invariant structure.}  
We introduce a regularization mechanism derived from the resonance constraint (Corollary~\ref{cor:torus_resonance}) that penalizes off-resonant Fourier modes and induces structured spectral sparsity.  
This provides a principled, representation-theoretic bias toward functions consistent with one-parameter subgroup invariance.

\item \textbf{Comprehensive empirical evaluation.}  
We validate the proposed framework across diverse synthetic and structured tasks, demonstrating that spectral sparsity reliably reveals latent one-parameter subgroups and improves generalization when incorporated into the learning process.

\end{itemize}

\section{Related Work}

 \subsection{Symmetries and Neural Network}
Symmetries and invariances arise naturally across physical, chemical, and biological systems, where they often determine fundamental structural and dynamical properties. Recent advances in areas such as drug discovery~\cite{stark2022equibind,corso2022diffdock,nguyen2025equicpi} and subatomic particle physics~\cite{gross1996role,hwang2010symmetries,gutowski2007symmetry} have demonstrated how explicitly leveraging symmetry can guide scientific reasoning 
and enable new discoveries. 

In deep learning, symmetry principles have become central to geometric deep learning~\cite{bronstein2021geometric}. For example, the success of convolutional neural networks (CNNs) in image processing~\cite{lecun1998gradient,sanborn2022bispectral} can largely be attributed to their built-in equivariance to two-dimensional translations~\cite{sanborn2022bispectral}. Building on this idea, recent works have incorporated broader group symmetries into neural architectures, including $G$-equivariant neural networks~\cite{cohen2016group}, steerable CNNs~\cite{weiler2019general,weiler2018learning}, and Lie group based constructions such as LieConv~\cite{finzi2020generalizing} and $SO(3)$-equivariant models~\cite{kondor2018clebsch}. These approaches demonstrate that encoding symmetry can significantly enhance representation efficiency and generalization. However, most of the existing methods assume explicit knowledge of the underlying symmetry group. 

In many scientific and real-world problems, the symmetry structure is implicit and not known a priori~\cite{yang2023latent,desai2022symmetry}, 
limiting the applicability of such approaches. This gap motivates the development of methods 
that can discover and exploit latent symmetries directly from data.    
\subsection{Automatic symmetry discovery}
Automatic symmetry discovery seeks to infer latent group structure directly from data symmetries~\cite{sanborn2022bispectral}. Early approaches parameterize Lie algebra generators to model continuous transformations~\cite{rao1998learning, chau2020disentangling, cohen2014learning}, but the recovered structure is often implicit and difficult to interpret. \citet{sanborn2022bispectral} employ spectral methods, though their framework is limited to discrete subgroups. Other methods, including LieGAN~\cite{yang2023liegan} and LieGG~\cite{moskalev2022liegg}, learn generators adversarially or extract them post hoc from trained networks, typically separating symmetry discovery from the predictive objective. Augerino~\cite{benton2020learning} enables joint learning of generators and predictors, yet the resulting invariance or equivariance of the model remains non-interpretable.

In contrast, our approach unifies prediction and symmetry discovery within a single training framework while ensuring that the predictor’s invariance is explicitly characterized. By embedding spectral analysis directly into learning, we obtain a transparent description of the underlying continuous group, so that both the discovered symmetry and the model’s invariance structure are analytically interpretable and aligned with the predictive task.

\section{Problem Formulation}
\begin{figure*}[htb!]
    \centering
    \includegraphics[width=0.8\linewidth, clip]{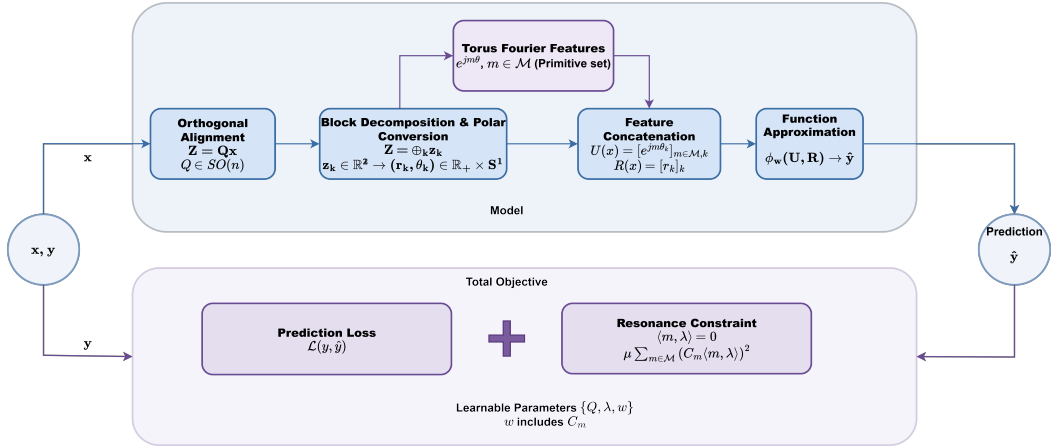}
    \caption{\textbf{Symmetry discovery and learning framework.}
    The input is first aligned via a learnable orthogonal transformation
    \(Q\in SO(n)\) and decomposed into two-dimensional blocks, which are
    converted to polar coordinates to obtain radii and torus angles.
    Primitive torus Fourier features \(U(x)\), together with radial features
    \(R(x)\), are fed to the predictor \(\phi_w\).
    Training minimizes the prediction loss along with a spectral resonance
    regularizer that promotes support only on frequency directions satisfying
    \(\langle m,\lambda\rangle = 0\), thereby enabling recovery of the latent
    one-parameter symmetry generator.}
    \label{fig:symmetry_block_diagram}
\end{figure*}






Let the dataset be 
\(
\mathcal{D} = \{(x_i, y_i)\}_{i=1}^{T},
\)
where each input \(x_i \in \mathbb{R}^n\), target \(y_i \in \mathbb{R}^m\) and \(SO(n)\) denote the special orthogonal group.

Assume the data are generated by an underlying function
\(
\vf : \mathcal{X} \rightarrow \mathbb{R}^m,
\)
so that \(y_i = \vf(x_i)\) for all \(i\), where the domain 
\(\gX \subset \sR^n\) is closed under the action of \(SO(n)\). Further, suppose that \(\vf\) is invariant under an unknown one-parameter subgroup 
\(\gH \subset SO(n)\), i.e.,
\begin{equation}
    \vf(h \cdot x) = \vf(x),
\quad \forall\, h \in \gH, \; x \in \gX. \nonumber
\label{eq:f-invariance}
\end{equation}

Given access only to the training dataset \(\gD\), 
the primary objectives of our work are:
\begin{enumerate}
    \item To learn a predictor  that accurately approximates 
    the true target function \(\vf\),
    \item To identify the unknown one-parameter subgroup \(\gH \subset SO(n)\) 
    under which \(\vf\) is invariant.
\end{enumerate}

To formalize the second objective, we briefly review the structure of one-parameter subgroups of \(SO(n)\).

\section{Preliminaries}

\subsection{One-Parameter Subgroups of \texorpdfstring{$SO(n)$}{SO(n)}}

A one-parameter subgroup of $SO(n)$ is a continuous family of rotations
\[
\gH = \{ \exp(tB) : t \in \mathbb{R} \},
\]
where $B \in \mathfrak{so}(n)$. The Lie algebra of $SO(n)$ is
\[
\mathfrak{so}(n)
=
\{ B \in \mathbb{R}^{n \times n} \mid B^\top = -B \},
\]
the space of skew-symmetric matrices.

Every $B \in \mathfrak{so}(n)$ generates a one-parameter subgroup via the matrix exponential, and conversely every one-parameter subgroup of $SO(n)$ arises in this way. Hence, learning a one-parameter subgroup of $SO(n)$ is equivalent to identifying its skew-symmetric generator.

\subsection{Canonical Structure of Skew-Symmetric Generators}

Every $B \in \mathfrak{so}(n)$ admits a real orthogonal block decomposition. 
There exists $Q \in SO(n)$ such that
\begin{equation}
Q^\top B Q
=
\bigoplus_{k=1}^{\lfloor n/2 \rfloor} \lambda_k J
\;\oplus\; \mathbf{0}_{\,n \bmod 2},
\qquad
J =
\begin{bmatrix}
0 & -1 \\
1 & \phantom{-}0
\end{bmatrix},
\label{eq:real-canonical}
\end{equation}
where $\lambda_k \in \mathbb{R}$.

Thus, any skew-symmetric generator decomposes into independent planar rotation generators acting on mutually orthogonal two-dimensional invariant subspaces. 
If $n$ is even, the decomposition consists entirely of $2\times2$ rotation blocks; if $n$ is odd, an additional one-dimensional trivial block appears. 
For simplicity, we assume $n$ is even.

Under this decomposition,
\[
\exp(tB)
=
Q
\left[
\bigoplus_{k=1}^{n/2}
R(\lambda_k t)
\right]
Q^\top,
\]
where $R(\theta)$ denotes a planar rotation matrix.


\noindent
\colorbox{gray!15}{
\parbox{0.97\linewidth}{
\textbf{Key Idea.}
Learning $B \;\Longleftrightarrow\;$ 
(i) learning an orthogonal matrix $Q$ that defines the invariant planes, and 
(ii) learning rotation rates $\lambda_1,\dots,\lambda_{n/2}$ governing the dynamics within each plane.
}}

\section{Methodology}

Building on the canonical structure described in the preliminaries, our objective is to jointly learn the target function $\vf$ and the latent one-parameter subgroup generated by $B$. Rather than optimizing directly over arbitrary skew-symmetric matrices, we parameterize the generator through its real canonical decomposition 
\[
B = Q \left( \bigoplus_{k=1}^{n/2} \lambda_k J \right) Q^\top,
\]
where the orthogonal matrix $Q$ determines the orientation of invariant two-dimensional planes and the parameters $\{\lambda_k\}$ specify the angular velocities within each plane. During training, $Q$ and $\{\lambda_k\}$ are learned end-to-end together with the predictor. Given an input $x$, we first transform it via $z = Q^\top x$, aligning the coordinates with the invariant planes so that the subgroup action decomposes into independent planar rotations.

Our key observation is that invariance under the subgroup generated by $B$ induces specific structure in the generalized Fourier transform (GFT) of the function expressed in this aligned domain. In particular, symmetry leads to spectral sparsity and resonance conditions among admissible frequencies determined by $\{\lambda_k\}$. Motivated by this structure, we construct Fourier features in the transformed domain and feed them into a function approximator $\phi$ such as neural network. To encourage recovery of the underlying symmetry, we introduce regularization that promotes sparsity and resonance-consistent frequency components. As training progresses, dominant spectral modes emerge whose structure characterizes the latent generator, thereby enabling simultaneous discovery of the symmetry parameters and efficient learning of the target function. Please refer to Figure~\ref{fig:symmetry_block_diagram} for further details about the architecture.

\medskip

\noindent
\colorbox{gray!15}{
\parbox{0.97\linewidth}{
\textbf{Key Idea.}
Continuous symmetries manifest as structured sparsity and resonance constraints in an appropriate Fourier domain. 
By aligning inputs with invariant planes through a learnable transformation, constructing Fourier features in this aligned basis, and enforcing sparsity and resonance-consistent structure in the spectral coefficients, we recover the latent generator from data. 
The dominant frequency components encode the underlying symmetry, while the resulting symmetry-adapted representation facilitates more stable and efficient function learning.
}}

\subsection{Spectral Characterization of Invariance}
To analyze the symmetry behavior of any function $f$ under different group action we first lift the function to the group $G < SO(n)$ such that for any fixed probe data $x_0 \in \mathbb{R}^n$ 
\[
F(g) := f(g\cdot x_0)
\]
Intuitively, for any given group transformation, the function $F$ captures the change in the target function for any given probe vector ($x_0$). 
If  $f$ is invariant under one-parameter subgroup $\gH < SO(n)$, then F(g) is also invariant under same group action i.e.,
\[
 F(h \cdot g) = F(\exp(tB) \cdot g) = F(g), \quad \forall h \in \gH, g \in G
\]

Therefore, symmetry of $f$ translates into left-invariance of $F$ along the subgroup direction. 
This invariance imposes strong structural constraints on the spectral decomposition of $F$, which we now analyze using the Group Fourier Transform.

\subsubsection{Group Fourier Transform}


Let $G$ be a compact Lie group equipped with normalized Haar measure $dg$. 
For an irreducible unitary representation 
$\rho_\pi : G \to \mathbb{C}^{d_\pi \times d_\pi}$, 
the Group Fourier Transform of $F \in L^2(G)$ at $\pi$ is defined as
\begin{equation}
\label{eq:gft_definition}
\widehat{F}(\pi)
=
\int_G F(g)\, \rho_\pi(g)^\ast \, dg,
\end{equation}
where $\widehat{F}(\pi) \in \mathbb{C}^{d_\pi \times d_\pi}$ denotes the matrix of Fourier coefficients associated with the representation $\pi$.

The collection $\{\widehat{F}(\pi)\}_{\pi \in \widehat{G}}$ provides a complete spectral description of $F$. 
Crucially, group invariance properties of $F$ translate into algebraic constraints on these matrix-valued Fourier coefficients. 
In particular, left-invariance along a one-parameter subgroup generated by $B$ imposes infinitesimal conditions involving the derived representation $d\rho_\pi(B)$.

The following proposition makes this relationship precise and forms the basis of our spectral characterization.

\begin{restatable}[Spectral characterization of left-invariance along $\exp(tB)$]{proposition}{propSpectralLeftInv}
\label{prop:spectral_left_invariance}
Let $G$ be a compact Lie group with normalized Haar measure $dg$ and Lie algebra $\fg$.
Fix $B\in\fg$ and define left-translation $(L_tF)(g):=F(\exp(tB)g)$.
For $F\in L^2(G)$ and an irreducible unitary representation
$\rho_\pi:G\to \sC ^{d_\pi\times d_\pi}$, define
\[
\widehat{F}(\pi):=\int_G F(g)\,\rho_\pi(g)^\ast\,dg
\quad \in \sC^{d_\pi\times d_\pi}.
\]
If $F$ is left-invariant along $\exp(tB)$, i.e.\ $L_tF=F$ for all $t\in\R$, then for every $\pi\in\Ghat$,
\begin{equation}
\label{eq:left_infinitesimal_condition_main}
\widehat{F}(\pi)\,d\rho_\pi(B)^\ast = 0.
\end{equation}
In particular, if $d\rho_\pi(B)$ is invertible then $\widehat{F}(\pi)=0$.
\end{restatable}

\noindent\textit{Interpretation.} ~\eqref{eq:left_infinitesimal_condition_main}
forces $\widehat{F}(\pi)$ to lie in the nullspace of $d\rho_\pi(B)^\ast$.

\subsubsection{GFT on the Maximal Torus}

While Proposition~\ref{prop:spectral_left_invariance} holds for a general compact Lie group $G$, in our analysis we specialize to the maximal torus of $SO(n)$. 
The maximal torus is a compact, abelian subgroup consisting of simultaneous planar rotations in mutually orthogonal invariant planes. 
This choice is particularly convenient because its irreducible representations are one-dimensional characters, leading to a simple and computationally efficient Fourier analysis. 
Importantly, any one-parameter subgroup of $SO(n)$, after alignment to its invariant planes via an appropriate orientation matrix $Q$, lies inside a maximal torus. 
Thus, performing the GFT on the maximal torus captures the spectral structure induced by any one-parameter subgroup, while significantly simplifying both the representation theory and the resulting computational framework.

We now specialize the general spectral constraint in Proposition~\ref{prop:spectral_left_invariance} to a choice of $G$ that yields a particularly simple and computationally efficient Fourier analysis. Concretely, we take $G$ to be a maximal torus of $SO(n)$ (with $n$ even for simplicity). After aligning coordinates to the invariant planes of the generator via an orientation matrix $Q$, the corresponding one-parameter subgroup $\{\exp(tB)\}$ acts as independent rotations within each $2$-dimensional plane and therefore lies inside a maximal torus. Since maximal tori are compact and abelian, all irreducible representations are one-dimensional characters, so the GFT reduces to a standard Fourier series on $\mathbb{T}^{n/2}$. This retains the spectral signature of any one-parameter subgroup (up to the learned orientation) while greatly simplifying the representation-theoretic machinery.

\paragraph{Torus parametrization and characters.}
Let $r:=n/2$. A maximal torus of $SO(n)$ can be identified (after a suitable choice of basis) with
\[
\mathbb{T}^r \cong (S^1)^r,
\qquad
\theta=(\theta_1,\dots,\theta_r)\in[0,2\pi)^r,
\]
where $\theta_k$ is the rotation angle in the $k$th invariant plane. The irreducible unitary representations  of $\mathbb{T}^r$ are the characters indexed by $m\in\mathbb{Z}^r$ ($\rho_\pi=\chi_m$):
\[
\chi_m(\theta) := e^{i\langle m,\theta\rangle}
= \exp\!\Big(i\sum_{k=1}^r m_k \theta_k\Big).
\]
Accordingly, the GFT coefficients of $F$ become scalars $\widehat{F}(m)\in\mathbb{C}$.

\begin{restatable}[Resonance condition on the maximal torus]{corollary}{corTorusResonance}
\label{cor:torus_resonance}
Assume $n$ is even and write the generator in canonical form
\[
B = Q\Big(\bigoplus_{k=1}^{r}\lambda_k J\Big)Q^\top,
\qquad r=n/2,
\]
with $\lambda=(\lambda_1,\dots,\lambda_r)\in\mathbb{R}^r$. 
Consider the maximal torus $\mathbb{T}^r$ in the aligned coordinates and let 
$\widehat{F}(m)$ denote the scalar Fourier coefficient of $F$ at character 
$\chi_m$, $m\in\mathbb{Z}^r$.

If $F$ is left-invariant along $\exp(tB)$, then
\begin{equation}
\label{eq:torus_resonance_main}
\widehat{F}(m)\,\langle m,\lambda\rangle = 0,
\qquad \forall m\in\mathbb{Z}^r.
\end{equation}
In particular, if $\langle m,\lambda\rangle \neq 0$ then $\widehat{F}(m)=0$.
Equivalently, nonzero spectral mass can occur only on the resonant set
\begin{equation}
  \mathcal{R}(\lambda) =  \{m\in\mathbb{Z}^r : \langle m,\lambda\rangle = 0\}.
  \label{eq:resonance-set}
\end{equation}
\end{restatable}

\noindent
\colorbox{gray!15}{
\parbox{0.97\linewidth}{
\textbf{Key Insight.}
Left-invariance along $\exp(tB)$ forces the Fourier spectrum to concentrate only on frequencies $m$ satisfying the linear resonance condition $\langle m,\lambda\rangle=0$. 
Thus, symmetry reduces the effective spectral support to a lower-dimensional subset, and the surviving dominant frequencies directly encode the generator parameters $\lambda$ (and, via alignment, the invariant planes).
}}

\begin{table*}[htb!]
\centering
\small
\setlength{\arrayrulewidth}{0.4pt}
\caption{\textbf{Comparison of symmetry discovery frameworks.}
We compare \emph{Spectral Discovery} (ours) with Augerino~\cite{benton2020learning},
LieGAN~\cite{yang2023liegan}, LieGG~\cite{moskalev2022liegg},
InfGen~\cite{ko2024learning}, BeyondAffine~\cite{shaw2024symmetry},
and Bispectral Neural Networks (BNNs)~\cite{sanborn2022bispectral}.
Methods are evaluated along four axes:
(i) whether symmetry discovery and prediction are learned in a \emph{joint training framework},
(ii) whether the resulting predictor is explicitly invariant in an interpretable manner,
(iii) the class of symmetry groups supported,
and (iv) whether the method is used as a baseline in this work.}
\label{tab:comparison}
\vspace{0.2cm}
\begin{tabular}{@{}lcccc@{}}
\toprule
{\color{black}\textbf{Method}}
  & {\color{black}\textbf{Joint Framework}}
  & {\color{black}\textbf{Interpretable Predictor}}
  & {\color{black}\textbf{Scope}}
  & {\color{black}\textbf{Baseline}} \\
\midrule

\rowcolor{ourrow}
\textbf{Ours: Spectral Discovery} & \textbf{Yes} & \textbf{Yes} & $1$D Lie groups & -- \\

\rowcolor{augerino} Augerino                          & \textbf{Yes} & No           & Generic Lie groups & \checkmark \\
\rowcolor{rowalt}
LieGAN                            & No (discovery only)      & No & Generic Lie groups & -- \\
LieGG                             & No (post-hoc discovery)  & No & Generic Lie groups & -- \\
\rowcolor{rowalt}
InfGen                            & No (post-hoc discovery)  & No & Generic Lie groups & -- \\
BeyondAffine                      & No (post-hoc discovery)  & No & $1$D Lie groups (incl.\ non-affine) & -- \\
\rowcolor{rowalt}
BNNs                              & \textbf{Yes}             & \textbf{Yes} & Discrete groups & -- \\

\bottomrule
\end{tabular}
\end{table*}
\subsection{Symmetry Discovery and Learning Framework}


The resonance condition in Corollary~\ref{cor:torus_resonance} provides the central principle behind our symmetry discovery framework. 
It shows that left-invariance along $\exp(tB)$ restricts the Fourier spectrum to frequencies satisfying the linear constraint $\langle m,\lambda\rangle = 0$. 
However, this characterization holds only when the generator is expressed in its canonical block-diagonal form, i.e., when the coordinate system is aligned with the invariant rotational planes. 
In general, these planes are unknown, and therefore the resonance condition alone reveals information about $\lambda$ only after an appropriate alignment.

\paragraph{Learning invariant planes via orthogonal alignment.}
To uncover the latent symmetry structure, we introduce a learnable orthogonal transformation 
$Q \in \mathbb{R}^{n \times n}$ that aligns the input space with the invariant planes of the generator. 
Since $B \in \mathfrak{so}(n)$ admits a block-diagonal decomposition into $2 \times 2$ rotational blocks, learning $Q$ allows us to represent the generator in this canonical form. 
We therefore transform the input features as
\begin{equation}
\label{eq:feature_transform}
Z(x) := Qx \in \mathbb{R}^n,
\end{equation}
where $n$ is assumed even for simplicity.

The transformed representation $Z(x)$ is decomposed into $n/2$ two-dimensional blocks:
\begin{equation}
\label{eq:block_decomposition}
Z(x) 
=
\bigoplus_{k=1}^{n/2} z_k,
\qquad
z_k = (z_{2k-1}, z_{2k})^\top \in \mathbb{R}^2,
\end{equation}
where $\bigoplus$ denotes concatenation of blocks. 
Each block $z_k$ corresponds to a candidate invariant rotational plane.

By learning the orthogonal matrix $Q$ jointly with the model parameters, we effectively search for a coordinate system in which the latent one-parameter subgroup acts as independent planar rotations. 
In this aligned domain, the resonance condition derived earlier becomes directly applicable for identifying the generator parameters.

\paragraph{Torus Fourier Features.}

Each two-dimensional block $z_k = (z_{2k-1}, z_{2k})^\top$ can be equivalently represented as a complex number
\[
u_k := z_{2k-1} + i z_{2k} \in \mathbb{C}.
\]
Under the aligned subgroup action, each block undergoes planar rotation,
\(
u_k \;\mapsto\; e^{i\lambda_k t} u_k,
\)
so that the action becomes diagonal in the complex representation.

Writing
\(
u_k = r_k e^{i\theta_k},
r_k = |u_k|,
\)
we obtain angular variables $\theta = (\theta_1,\dots,\theta_{n/2})$ that parametrize a maximal torus.

Motivated by the GFT analysis on this torus, we construct Fourier features indexed by integer frequency vectors 
$m = (m_1,\dots,m_{n/2}) \in \mathbb{Z}^{n/2}$:
\begin{equation}
\label{eq:fourier_feature}
\widehat{Z}_m(x)
=
\exp \big(i \langle m,\theta(x)\rangle\big)
=
\exp \left(
i \sum_{k=1}^{n/2} m_k \theta_k(x)
\right). \nonumber
\end{equation}
These correspond exactly to the characters of the maximal torus and form the spectral basis in which the resonance condition 
\[
\langle m,\lambda\rangle = 0
\]
arises naturally.

\paragraph{Bandwidth Truncation and Fundamental Frequencies.}

In practice, we restrict attention to a finite bandwidth parameter $\mathcal{B}$ such that
\[
m_k \in \{-\mathcal{B}, \dots, \mathcal{B}\},
\]
which yields $(2\mathcal{B}+1)^{n/2}$ candidate frequency components. 
However, many of these frequencies lie along the same lattice direction and therefore encode redundant symmetry information.

\paragraph{Primitive Directions and Frequency Rays.}

The index set $\mathbb{Z}^{n/2}$ forms an integer lattice. 
Each nonzero index $m$ determines a one-dimensional frequency direction, since all integer multiples of $m$ correspond to harmonics of the same angular combination.

To isolate the fundamental direction associated with $m$, we define its \emph{primitive representative}:
\begin{equation}
\label{eq:primitive_direction}
\mathrm{prim}(m)
=
\frac{m}{\gcd(|m_1|,\dots,|m_{n/2}|)}.
\end{equation}
By construction, $\mathrm{prim}(m)$ has coprime entries and is the minimal integer vector in the same direction as $m$.

We therefore retain only the set of primitive frequency directions:
\begin{equation}
\label{eq:primitive_set}
\mathcal{M}
=
\left\{
\mathrm{prim}(m)
\;:\;
m_k \in \{-\mathcal{B},\dots,\mathcal{B}\},\;
m \neq 0
\right\}.
\end{equation}

\paragraph{Model Architecture and Objective.}

Let $U(x) = \{\widehat{Z}_m(x)\}_{m \in \mathcal{M}}$ denote the collection of primitive Fourier features and let 
$R(x) = (r_1,\dots,r_{n/2})$ denote the radii of the aligned 2D blocks. 
We concatenate these features and feed them to a function approximator
\[
\phi_w : (U(x), R(x)) \mapsto \widehat{y},
\]
parameterized by weights $w$.

In addition to learning $w$ and the alignment matrix $Q$, 
we learn the generator parameters $\lambda$ by enforcing the resonance condition derived in Corollary~\ref{cor:torus_resonance}. 
Recall that symmetry requires nonzero Fourier coefficients only when $\langle m,\lambda\rangle = 0$. 
To promote this structure, we penalize violations of resonance weighted by the corresponding feature coefficients.

Let $C_m$ denote the first-layer weight in $\phi_w$ associated with feature $\widehat{Z}_m$. 
The overall objective is

\noindent
\colorbox{gray!15}{
\parbox{0.97\linewidth}{
\begin{equation}
\label{eq:final_objective}
\min_{w,\,\lambda,\,Q}
\;
\mathcal{L}\!\left(y, \phi_w(U(x), R(x))\right)
\;+\;
\mu
\sum_{m \in \mathcal{M}}
\left(
C_m \,\langle m,\lambda\rangle
\right)^2,
\end{equation}
}}

where $\mu > 0$ controls the strength of the resonance regularization.

The first term ensures predictive accuracy, while the second term encourages spectral support to concentrate on resonant frequency directions. 
After training, the learned parameters $(Q,\lambda)$ determine the estimated Lie algebra generator
\[
B = Q \left( \bigoplus_{k=1}^{n/2} \lambda_k J \right) Q^\top,
\]
and $\phi_w$ provides the final prediction.

\section{Experiments}
\paragraph{Objective.}
We evaluate the proposed spectral discovery framework for jointly learning a predictive mapping and recovering the latent one-parameter continuous symmetry underlying the data. Our emphasis is on accurately identifying both the symmetry orientation and the subgroup direction while maintaining strong predictive performance.

\paragraph{Benchmarks.}
We consider two representative tasks spanning synthetic physical dynamics and real-world data. 
First, we evaluate on a 6D double pendulum system with spring coupling, where the target mapping is invariant under a latent one-parameter rotation acting simultaneously across multiple 2D planes. This task assesses recovery of a shared subgroup direction in the presence of nuisance coordinates. 
Second, we consider the Top Quark Tagging classification task using jet constituent four-momenta. The underlying physical laws admit Lorentz symmetries, and we focus on discovering a rotational one-parameter subgroup while learning the classification mapping. Ablations on  noise robustness (Figure~\ref{fig:noise_sweep}) and sample efficiency (Figure~\ref{fig:sample_sweep}) are provided in Appendix~\ref{sup_sec:sensitivity_analysis}.

\paragraph{Baseline.}
We compare against \textbf{Augerino}~\citep{benton2020learning}, which jointly learns a predictive model and a distribution over continuous transformations to encourage invariance during training (see Table~\ref{tab:comparison} for further details). This provides the closest comparison in a joint supervised learning and symmetry discovery setting. Complete training details are provided in Appendix~\ref{sup_sec:training_protocol}. Comparison with additional baselines is provided in Appendix~\ref{sup_sec:comp_lieGAN}.

\begin{figure*}[htb!]
        \centering
    \includegraphics[width=\textwidth, trim={0, 0, 0, 0}, clip]{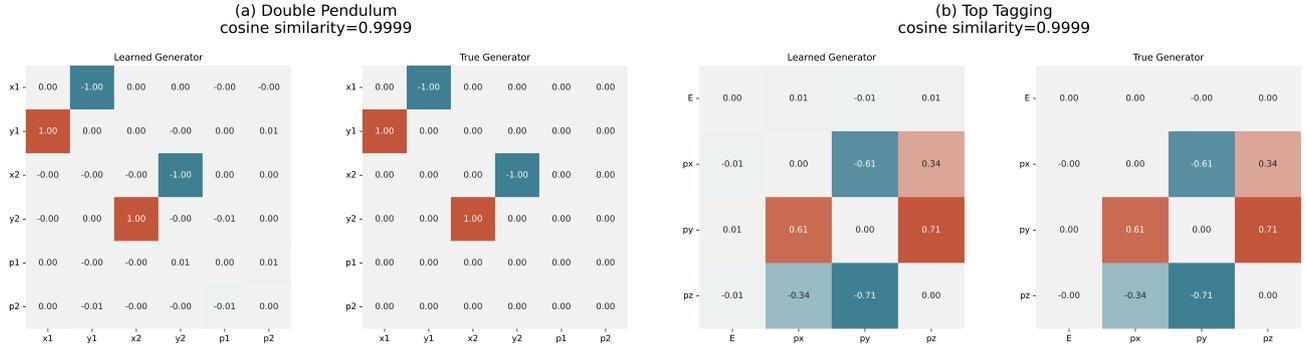}
    \caption{
\textbf{Recovered vs. True Rotational Generators.} Comparison of the learned (left) and ground-truth (right) generators for \textbf{(a)} the Double Pendulum system (diagonal $\Delta(SO(2))$ generator) and \textbf{(b)} the Top Tagging task. Both learned generators demonstrate near-perfect alignment with the physical ground truth (cosine similarity $= 0.9999$). 
    }
    \label{fig:double_pendulum_top_tagging_generator}
\end{figure*}


\begin{table*}[t]
\centering
\small
\caption{
\textbf{Double Pendulum (6D), 32K samples.}
Comparison of predictive accuracy and symmetry recovery.
Results are reported as mean $\pm$ standard deviation.
Lower is better for Test MSE and Invariance Error; higher is better for Cosine Similarity.
}
\label{tab:double_pendulum}
\begin{tabular}{lccc}
\toprule
\textbf{Method} 
& \textbf{Test MSE} $\downarrow$ 
& \textbf{Inv.\ Error} $\downarrow$ 
& \textbf{Cosine Similarity} $\uparrow$ \\
\midrule
\rowcolor{ourrow_tab} \textbf{Spectral Discovery} 
& \textbf{0.00298 $\pm$ 0.00024}
& \textbf{0.00070 $\pm$ 0.00027} 
& \textbf{0.9999 $\pm$ 0.00001} \\

\rowcolor{augerino_tab} Augerino
& $0.01053 \pm 0.00040$ 
& $0.00232 \pm 0.00025$
& $0.9233 \pm 0.0708$ \\
\bottomrule
\end{tabular}
\end{table*}

\begin{table}[t]
\small
\centering
\caption{
\textbf{Top Quark Tagging (Elliptic Regime).}
Classification accuracy and generator recovery.
Results are reported as mean $\pm$ standard deviation.
Higher is better for both metrics.
}
\label{tab:top_tagging_results}
\begin{tabular}{lcc}
\toprule
\textbf{Method} 
& \textbf{Accuracy (\%)} $\uparrow$ 
& \textbf{Cosine Similarity} $\uparrow$ \\
\midrule
\rowcolor{ourrow_tab} \textbf{Spectral Discovery}
& \textbf{84.87 $\pm$ 0.39} 
& \textbf{0.9999 $\pm$ 0.00002} \\

\rowcolor{augerino_tab} Augerino
& $74.9 \pm 2.4$ 
& $0.994 \pm 0.003$ \\

\bottomrule
\end{tabular}
\end{table}

\label{sec:experiments}

\paragraph{Evaluation Metrics}
\label{sec:metrics}
We evaluate performance using three metrics: (M1) \emph{Test MSE}, computed on a held-out test set; (M2) \emph{Invariance Error}, (M3) Cosine similarity between the learned and true generators. Further details about  evaluation metrics are provided in the Appendix ~\ref{sup_sec:metrics}.





\subsection{Results}

\paragraph{Double Pendulum (6D).}
On the double pendulum benchmark (Figure~\ref{fig:double_pendulum_top_tagging_generator} and Table~\ref{tab:double_pendulum}), Spectral Discovery achieves substantially stronger symmetry recovery while maintaining competitive predictive accuracy. The learned generator aligns nearly perfectly with the ground-truth rotational symmetry, with significantly lower invariance error than Augerino. While both methods fit the regression mapping, augmentation-based invariance does not reliably recover the true infinitesimal generator, whereas spectral discovery yields stable and interpretable generator identification.

\paragraph{Top Quark Tagging.}
On the Top Quark Tagging task (Table~\ref{tab:top_tagging_results}), Spectral Discovery improves classification accuracy while consistently recovering the underlying rotational subgroup, as evidenced by high generator alignment. These results demonstrate that explicit symmetry discovery in the spectral domain can jointly enhance predictive performance and structural interpretability in real-world high-energy physics settings.

\section{Discussion}
\label{subsec:connection_emlp}


An important observation emerging from our analysis is the structural similarity between our formulation and Equivariant Multi-Layer Perceptrons (EMLP)~\cite{finzi2021practical}. For a known symmetry group, EMLP enforces equivariance via the fixed-point condition $\rho(g)v = v$, which for continuous groups reduces to the Lie algebra constraint $d\rho(A)v = 0$ for all $A \in \mathfrak{g}$. 
To implement this, EMLP constructs equivariant linear layers by vectorizing the weights as $v = \mathrm{vec}(W)$ in a tensor-product representation space  (e.g., $V_{\mathrm{out}} \otimes V_{\mathrm{in}}^*$) 
and obtaining equivariant solutions as elements of the nullspace of the induced 
Lie algebra action. In our framework, each Fourier coefficient 
\[
\widehat F(\pi)=\int_G F(g)\rho_\pi(g)^*\,dg
\in \operatorname{End}(V_\pi)\cong V_\pi\otimes V_\pi^*
\]
resides in the same type of tensor-product space. The infinitesimal constraint 
\(
\widehat F(\pi)\, d\rho_\pi(B)^* = 0
\)
imposes an identical Lie-algebra nullspace condition. Thus, both approaches reduce symmetry to nullspace constraints in representation space: EMLP constrains vectorized weights, while our method constrains Fourier blocks. When a layer maps $V_\pi \to V_\pi$, its weight lies in $\operatorname{End}(V_\pi)$, 
the same space as the corresponding GFT coefficient, making the latter one of the valid solutions of the weight.
\paragraph{Limitations and Future Work}
Our framework currently targets one-parameter subgroups of $SO(n)$ and is optimized for recovering a single dominant generator. Extending it to jointly discover higher-dimensional or multiple commuting subgroups remains future work. Moreover, the analysis is performed on the maximal torus of $SO(n)$, leveraging its simple irreducible representations for efficiency. Incorporating broader compact-group representations beyond the maximal torus would expand the range of discoverable symmetries.

\bibliography{uai2026-template}

\newpage

\onecolumn

\title{Spectral Discovery of Continuous Symmetries via Generalized Fourier Transforms\\(Supplementary Material)}
\maketitle

\appendix

\begin{table}[htbp!]
\centering
\small
\setlength{\tabcolsep}{6pt}
\renewcommand{\arraystretch}{1.15}
\caption{\textbf{Notation used in the method.}}
\label{tab:notation}
\begin{tabular}{@{}p{0.28\linewidth} p{0.66\linewidth}@{}}
\toprule
\textbf{Symbol} & \textbf{Meaning} \\
\midrule
$\gD=\{(x_i,y_i)\}_{i=1}^{T}$ & Training dataset of $T$ input--target pairs. \\
$T$ & Number of training samples. \\
$x_i \in \R^n$ & Input feature vector for sample $i$. \\
$y_i \in \R^m$ & Target/output vector for sample $i$. \\
$n,m$ & Input and output dimensions, respectively. \\
$\gX \subset \R^n$ & Input domain (assumed closed under the action of $SO(n)$). \\
$\vf:\gX \to \R^m$ & Unknown ground-truth target function generating $y=\vf(x)$. \\
$SO(n)$ & Special orthogonal group (rotations in $\R^n$). \\
$\mathfrak{so}(n)$ & Lie algebra of $SO(n)$: skew-symmetric matrices $B^\top=-B$. \\
$\gH \subset SO(n)$ & Unknown one-parameter subgroup under which $\vf$ is invariant. \\
$h \cdot x$ & Group action of $h \in SO(n)$ on $x \in \R^n$ (rotation). \\
$\gH=\{\exp(tB):t\in\R\}$ & One-parameter subgroup generated by $B\in\mathfrak{so}(n)$. \\
$B$ & Skew-symmetric Lie-algebra generator of the latent symmetry. \\
$\exp(tB)$ & Matrix exponential mapping generator $B$ to a group element in $SO(n)$. \\
$Q\in SO(n)$ & Learnable orthogonal alignment matrix (orients invariant $2$D planes). \\
$J=\begin{bmatrix}0&-1\\ 1&0\end{bmatrix}$ & Canonical $2\times 2$ rotation generator. \\
$r=n/2$ & Number of invariant $2$D planes (assuming $n$ even). \\
$\lambda=(\lambda_1,\dots,\lambda_r)\in\R^r$ & Rotation rates / angular velocities in each invariant plane. \\
$B = Q\big(\bigoplus_{k=1}^r \lambda_k J\big)Q^\top$ & Real canonical decomposition/parameterization of the generator. \\
$z = Q^\top x$ or $Z(x)=Qx$ & Aligned coordinates after orthogonal transformation (notation varies in text). \\
$Z(x)=\bigoplus_{k=1}^{r} z_k$ & Block decomposition of aligned coordinates into $r$ two-dimensional blocks. \\
$z_k=(z_{2k-1},z_{2k})^\top\in\R^2$ & The $k$-th $2$D block (candidate invariant plane coordinates). \\
$u_k=z_{2k-1}+ i z_{2k}\in\sC$ & Complex representation of the $k$-th $2$D block. \\
$u_k=r_k e^{i\theta_k}$ & Polar form: radius $r_k$ and angle $\theta_k$. \\
$r_k=|u_k|$ & Radius (magnitude) of block $k$. \\
$\theta_k \in [0,2\pi)$ & Torus angle (phase) in block $k$. \\
$\theta=(\theta_1,\dots,\theta_r)$ & Vector of torus angles; parameterizes $\sT^r\cong(S^1)^r$. \\
$\sT^r$ & Maximal torus in aligned coordinates (product of $r$ circles). \\
$m=(m_1,\dots,m_r)\in\sZ^r$ & Integer frequency (character index) on the torus. \\
$\langle m,\theta\rangle$ & Standard inner product $\sum_{k=1}^r m_k\theta_k$. \\
$\chi_m(\theta)=e^{i\langle m,\theta\rangle}$ & Torus character / Fourier basis function. \\
$\widehat{Z}_m(x)=\exp(i\langle m,\theta(x)\rangle)$ & Torus Fourier feature indexed by $m$. \\
$\mathcal{B}$ & Bandwidth truncation; restricts $m_k\in\{-\mathcal{B},\dots,\mathcal{B}\}$. \\
$\mathrm{prim}(m)$ & Primitive direction: $m/\gcd(|m_1|,\dots,|m_r|)$. \\
$\mathcal{M}$ & Set of primitive frequency directions retained after bandwidth truncation. \\
$U(x)=\{\widehat{Z}_m(x)\}_{m\in\mathcal{M}}$ & Collection of primitive torus Fourier features. \\
$R(x)=(r_1,\dots,r_r)$ & Vector of radii from aligned $2$D blocks. \\
$\phi_w$ & Predictor / function approximator with parameters (weights) $w$. \\
$\widehat{y}=\phi_w(U(x),R(x))$ & Model prediction. \\
$\mathcal{L}(y,\phi_w(\cdot))$ & Prediction loss (task-dependent). \\
$C_m$ & First-layer weight in $\phi_w$ associated with Fourier feature indexed by $m$. \\
$\mu>0$ & Regularization strength for resonance penalty. \\
$\mathcal{R}(\lambda)=\{m\in\sZ^r:\langle m,\lambda\rangle=0\}$ & Resonant frequency set implied by invariance. \\
$\langle m,\lambda\rangle=0$ & Resonance condition (nonzero spectral mass only on resonant frequencies). \\
\bottomrule
\end{tabular}
\end{table}

\section{Identifiability of the Alignment and Generator Parameters}
\label{sec:identifiability}

Our model parameterizes a one-parameter subgroup generator in canonical form
\begin{equation}
B \;=\; Q \Big(\bigoplus_{k=1}^{r} \lambda_k J \Big) Q^\top,
\qquad
J \;=\;
\begin{pmatrix}
0 & -1\\
1 & 0
\end{pmatrix},
\qquad
r = n/2,
\label{eq:B_canonical}
\end{equation}
where $Q\in SO(n)$ aligns the ambient coordinates with $r$ invariant $2$-planes and
$\lambda=(\lambda_1,\dots,\lambda_r)\in\mathbb{R}^r$ are the rotation rates within each plane.
Since $Q$ and $\lambda$ are learned from data through spectral structure, it is important to clarify
what aspects of $(Q,\lambda)$ are identifiable.

\subsection{Identifiability of $Q$ up to block-wise rotations and permutations}
\label{subsec:Q_ident}

The decomposition in~\eqref{eq:B_canonical} is not unique: each invariant $2$-plane admits an
arbitrary choice of orthonormal basis. Concretely, let
\begin{equation}
S \;=\; \bigoplus_{k=1}^{r} R(\alpha_k),
\qquad
R(\alpha) \;=\;
\begin{pmatrix}
\cos\alpha & -\sin\alpha\\
\sin\alpha & \cos\alpha
\end{pmatrix}
\in SO(2),
\label{eq:block_rotation_S}
\end{equation}
be a block-diagonal change of basis that rotates each $2$-plane independently.
Because $R(\alpha)J R(\alpha)^\top = J$ for all $\alpha$, we have
\begin{equation}
S\Big(\bigoplus_{k=1}^{r}\lambda_k J\Big)S^\top
\;=\;
\bigoplus_{k=1}^{r}\lambda_k J.
\label{eq:S_commutes}
\end{equation}
Therefore, replacing $Q$ with $Q' = QS$ leaves $B$ unchanged:
\begin{equation}
Q'\Big(\bigoplus_{k=1}^{r}\lambda_k J\Big)Q'^\top
\;=\;
Q\Big(\bigoplus_{k=1}^{r}\lambda_k J\Big)Q^\top
\;=\; B.
\label{eq:Q_block_rotation_equiv}
\end{equation}
Hence, $Q$ is identifiable only up to right-multiplication by block-diagonal rotations~\eqref{eq:block_rotation_S}
(i.e., a gauge freedom corresponding to basis choice within each invariant plane).

In addition, if we permute the ordering of the invariant planes, the generator is unchanged up to a corresponding
permutation of $\lambda$. Let $P$ be a $2\times2$-block permutation matrix (i.e., $P$ permutes the $r$ blocks),
so that
\begin{equation}
P\Big(\bigoplus_{k=1}^{r}\lambda_k J\Big)P^\top
\;=\;
\bigoplus_{k=1}^{r}\lambda_{\pi(k)} J
\label{eq:block_perm}
\end{equation}
for some permutation $\pi$ of $\{1,\dots,r\}$. Then $B$ is also invariant under
\begin{equation}
(Q,\lambda) \;\mapsto\; (QP,\,\pi(\lambda)).
\label{eq:perm_equiv}
\end{equation}
Combining the two, the alignment matrix is identifiable only up to the equivalence class
\begin{equation}
Q \;\sim\; QSP,
\label{eq:Q_equiv_class}
\end{equation}
where $S$ is block-diagonal with $SO(2)$ blocks and $P$ is a $2$-block permutation matrix.
Finally, if some rates coincide (e.g., $\lambda_i=\lambda_j$), then additional mixing between the corresponding
$2$-planes may be possible without changing $B$, enlarging the non-identifiability group beyond~\eqref{eq:Q_equiv_class}.
In practice, we interpret successful recovery of $Q$ as recovery of the invariant \emph{subspaces} (planes), rather
than a unique orthonormal basis within each plane.

\subsection{Identifiability of $\lambda$ from surviving (resonant) frequencies}
\label{subsec:lambda_ident}

On the maximal torus in aligned coordinates, invariance along $\exp(tB)$ implies the resonance constraint
\begin{equation}
\widehat{F}(m)\,\langle m,\lambda\rangle = 0,
\qquad m\in\mathbb{Z}^{r},
\label{eq:resonance_constraint}
\end{equation}
so nonzero spectral mass can occur only on frequency indices $m$ satisfying
\begin{equation}
\langle m,\lambda\rangle = 0.
\label{eq:resonant_set}
\end{equation}
Let $\mathcal{M}_{\mathrm{surv}} \subset \mathbb{Z}^{r}$ denote the set of \emph{surviving} (detected) frequency
directions, e.g., those for which the learned spectral coefficients are non-negligible.
Then $\lambda$ must satisfy the homogeneous linear system
\begin{equation}
\langle m,\lambda\rangle = 0, \qquad \forall m\in \mathcal{M}_{\mathrm{surv}}.
\label{eq:lambda_linear_system}
\end{equation}
Stacking constraints yields
\begin{equation}
M\lambda = 0,
\qquad
M \in \mathbb{Z}^{|\mathcal{M}_{\mathrm{surv}}|\times r},
\label{eq:M_matrix}
\end{equation}
where each row of $M$ is $m^\top$ for $m\in\mathcal{M}_{\mathrm{surv}}$.

\paragraph{Existence.}
The system~\eqref{eq:M_matrix} always admits the trivial solution $\lambda=0$; however $\lambda=0$ corresponds to
a degenerate subgroup action. In our setting, we seek nontrivial $\lambda\neq 0$, which exists iff
\begin{equation}
\dim\ker(M) \;\ge\; 1.
\label{eq:existence}
\end{equation}
Equivalently, $\mathrm{rank}(M) \le r-1$. In practice, approximate invariance leads to approximate constraints
$\langle m,\lambda\rangle \approx 0$, and $\lambda$ can be estimated as a (regularized) right-singular vector
associated with the smallest singular value of $M$.

\paragraph{Uniqueness up to scale (and sign).}
Because~\eqref{eq:lambda_linear_system} is homogeneous, if $\lambda$ is feasible then so is $c\lambda$ for any
$c\in\mathbb{R}$, so $\lambda$ can be at best identifiable up to scale (and thus sign).
A standard normalization (e.g., $\|\lambda\|_2=1$) fixes this degree of freedom.

After normalization, $\lambda$ is unique iff the nullspace is one-dimensional:
\begin{equation}
\dim\ker(M) = 1
\quad\Longleftrightarrow\quad
\mathrm{rank}(M) = r-1.
\label{eq:unique_condition}
\end{equation}
This condition depends on the \emph{number} and \emph{linear independence} of surviving frequency directions.
Intuitively, each independent $m$ contributes one linear constraint on $\lambda$; to determine $\lambda$ up to
scale in $\mathbb{R}^r$, we need $r-1$ independent constraints.

\paragraph{Role of the types of surviving frequencies.}
If the surviving set $\mathcal{M}_{\mathrm{surv}}$ lies in a low-dimensional subspace (e.g., many $m$ are collinear
or concentrated on a few coordinates), then $\mathrm{rank}(M)$ may be small and $\ker(M)$ high-dimensional, yielding
non-identifiability. Conversely, a diverse set of linearly independent $m$ vectors tightens the constraints and can
make $\lambda$ identifiable (up to scale). This motivates using \emph{primitive} frequency directions and promoting
sparsity patterns that reveal multiple independent resonant $m$'s.

\paragraph{Practical estimation.}
Given an estimated $\mathcal{M}_{\mathrm{surv}}$, one may recover $\lambda$ by solving
\begin{equation}
\min_{\lambda\in\mathbb{R}^r}\ \|M\lambda\|_2^2
\quad \text{s.t.}\quad \|\lambda\|_2=1,
\label{eq:lambda_estimation}
\end{equation}
which returns the unit-norm right-singular vector of $M$ associated with the smallest singular value. When the
constraints are exact and $\mathrm{rank}(M)=r-1$, this recovers $\lambda$ uniquely up to sign.

\subsection{Summary of identifiability}
Overall, the generator $B$ is identifiable, while $(Q,\lambda)$ are identifiable only up to natural symmetries:
(i) $Q$ up to independent $SO(2)$ rotations within each invariant plane and permutations of the planes, and
(ii) $\lambda$ up to scale/sign, with uniqueness after normalization governed by the rank condition
$\mathrm{rank}(M)=r-1$ determined by the number and linear independence of surviving resonant frequencies.

\paragraph{What ultimately matters.}
While the alignment matrix $Q$ is only identifiable up to block-wise rotations and permutations
(Section~\ref{subsec:Q_ident}), these ambiguities do \emph{not} affect the recovered one-parameter
subgroup. Indeed, all equivalent choices of $(Q,\lambda)$ related by the transformations in
\eqref{eq:Q_equiv_class} and scaling of $\lambda$ induce the \emph{same} Lie algebra generator $B$
in~\eqref{eq:B_canonical} and therefore the same subgroup $\{\exp(tB)\}$.

Consequently, the practically meaningful question is whether the underlying generator direction is
correctly identified. This reduces primarily to identifiability of $\lambda$. If there exists a
nontrivial solution $\lambda\neq 0$ to~\eqref{eq:lambda_linear_system} and the nullspace is
one-dimensional (so $\lambda$ is unique up to scale/sign), then the induced generator
\begin{equation}
B \;=\; Q\Big(\bigoplus_{k=1}^{r}\lambda_k J\Big)Q^\top
\end{equation}
is uniquely determined (up to the inherent gauge symmetries above) and correctly captures the
underlying one-parameter subgroup.

In summary, ambiguities in $Q$ correspond only to basis choices within invariant planes and are
largely irrelevant for symmetry discovery. The central requirement for successful recovery is that
$\lambda$ be identifiable (up to scaling), which in turn depends on the number and linear
independence of the surviving resonant frequency directions.

\section{Proofs}
\propSpectralLeftInv*
\begin{proof}
Fix an irrep $\rho_\pi:G\to \sC^{d_\pi\times d_\pi}$ and write $u_t:=\exp(tB)\in G$.
We first compute how the group Fourier coefficient transforms under left translation.

\paragraph{Step 1: Fourier coefficient of a left translate.}
By definition,
\[
\widehat{L_tF}(\pi)
=\int_G (L_tF)(g)\,\rho_\pi(g)^\ast\,dg
=\int_G F(u_t g)\,\rho_\pi(g)^\ast\,dg.
\]
Use the change of variables $h=u_t g$ (equivalently, $g=u_t^{-1}h$). Since Haar measure is left-invariant, $dg=dh$.
Thus
\begin{align*}
\widehat{L_tF}(\pi)
&=\int_G F(h)\,\rho_\pi(u_t^{-1}h)^\ast\,dh \\
&=\int_G F(h)\,\bigl(\rho_\pi(u_t^{-1})\,\rho_\pi(h)\bigr)^\ast\,dh \\
&=\int_G F(h)\,\rho_\pi(h)^\ast\,\rho_\pi(u_t^{-1})^\ast\,dh \\
&=\widehat{F}(\pi)\,\rho_\pi(u_t^{-1})^\ast.
\end{align*}
Because $\rho_\pi$ is unitary, $\rho_\pi(u_t^{-1})=\rho_\pi(u_t)^\ast$, hence
$\rho_\pi(u_t^{-1})^\ast=\rho_\pi(u_t)$. Therefore,
\begin{equation}
\label{eq:hat_left_translate}
\widehat{L_tF}(\pi)=\widehat{F}(\pi)\,\rho_\pi(u_t)
\qquad \forall\,t\in\R.
\end{equation}

\paragraph{Step 2: Use left-invariance.}
If $F$ is left-invariant along $\exp(tB)$, then $L_tF=F$ for all $t$, hence
$\widehat{L_tF}(\pi)=\widehat{F}(\pi)$ for all $t$. Combining with \eqref{eq:hat_left_translate},
\begin{equation}
\label{eq:fixed_by_one_parameter}
\widehat{F}(\pi)\,\rho_\pi(u_t)=\widehat{F}(\pi)\qquad \forall\,t\in\R.
\end{equation}
Equivalently,
\begin{equation}
\label{eq:annihilate}
\widehat{F}(\pi)\,\bigl(\rho_\pi(u_t)-I\bigr)=0\qquad \forall\,t\in\R.
\end{equation}

\paragraph{Step 3: Differentiate at $t=0$.}
The map $t\mapsto \rho_\pi(u_t)=\rho_\pi(\exp(tB))$ is smooth (indeed analytic), so we may differentiate
\eqref{eq:fixed_by_one_parameter} at $t=0$:
\[
0=\frac{d}{dt}\Big|_{t=0}\Bigl(\widehat{F}(\pi)\,\rho_\pi(\exp(tB))\Bigr)
=\widehat{F}(\pi)\,\frac{d}{dt}\Big|_{t=0}\rho_\pi(\exp(tB))
=\widehat{F}(\pi)\,d\rho_\pi(B),
\]
where $d\rho_\pi(B):=\frac{d}{dt}\big|_{t=0}\rho_\pi(\exp(tB))$ is the derived (Lie algebra) representation.

Finally, since $\rho_\pi$ is unitary, $d\rho_\pi(B)$ is skew-Hermitian:
\[
\rho_\pi(\exp(tB))^\ast \rho_\pi(\exp(tB))=I
\ \Rightarrow\
\frac{d}{dt}\Big|_{t=0}\bigl(\rho_\pi(\exp(tB))^\ast \rho_\pi(\exp(tB))\bigr)=0
\ \Rightarrow\
d\rho_\pi(B)^\ast + d\rho_\pi(B)=0.
\]
Hence $d\rho_\pi(B)^\ast=-d\rho_\pi(B)$, and the identity $\widehat{F}(\pi)\,d\rho_\pi(B)=0$
is equivalent to
\[
\widehat{F}(\pi)\,d\rho_\pi(B)^\ast=0,
\]
which is exactly \eqref{eq:left_infinitesimal_condition_main}.

\paragraph{Step 4: Invertible case.}
If $d\rho_\pi(B)$ is invertible, then multiplying $\widehat{F}(\pi)\,d\rho_\pi(B)^\ast=0$
on the right by $(d\rho_\pi(B)^\ast)^{-1}$ gives $\widehat{F}(\pi)=0$.

This holds for every $\pi\in\widehat{G}$, completing the proof.
\end{proof}

\corTorusResonance*
\begin{proof}
We apply Proposition~\ref{prop:spectral_left_invariance} with $G=\mathbb{T}^r$, whose irreducible unitary representations are the characters $\rho_\pi=\chi_m$.
Left-invariance along $\exp(tB)$ means $F(\exp(tB)g)=F(g)$ for all $t$, hence Proposition~\ref{prop:spectral_left_invariance} gives
\[
\widehat{F}(m)\, d\chi_m(B)^\ast = 0.
\]
It remains to compute $d\chi_m(B)$. In the aligned coordinates, $\exp(tB)$ corresponds to torus angles $\theta(t)=(\lambda_1 t,\dots,\lambda_r t)$, so
\[
\chi_m(\exp(tB)) = \chi_m(\theta(t)) = \exp\!\big(i\langle m,\lambda\rangle t\big).
\]
Differentiating at $t=0$ yields
\begin{align*}
    d\chi_m(B) = \left.\frac{d}{dt}\chi_m(\exp(tB))\right|_{t=0}
&= i\langle m,\lambda\rangle, \\
d\chi_m(B)^\ast &= -i\langle m,\lambda\rangle.
\end{align*}
Substituting into the spectral constraint gives
\[
\widehat{F}(m)\,(-i\langle m,\lambda\rangle)=0
\quad\Longleftrightarrow\quad
\widehat{F}(m)\,\langle m,\lambda\rangle=0,
\]
Thus, nonzero Fourier coefficients $\widehat{F}(m)$ occur, only if $\langle m,\lambda\rangle=0$. 
\end{proof}

\section{Experiments}

\subsection{Training Protocol}
\label{sup_sec:training_protocol}


\paragraph{Datasets.}
For the double pendulum experiment, we use $32{,}000$ samples.
For the Top Quark Tagging task, we use $64{,}000$ samples.
In both cases, we follow the same training protocol described below.

\paragraph{Model architecture.}
The predictor $\phi_w$ is implemented as a multilayer perceptron (MLP) with three hidden layers and ReLU activations.
Unless otherwise stated, the spectral and radial features described in Section~5 are concatenated and fed into this network.

\paragraph{Optimization.}
All models are trained for $40$ epochs using the Adam optimizer with learning rate $2\times 10^{-3}$.
Mini-batch training is used throughout.

\paragraph{Resonance regularization schedule.}
To stabilize training, we employ a warm-up phase for the resonance regularization coefficient $\mu$ in the objective~\eqref{eq:final_objective}.
For the first $10$ epochs, $\mu$ is set to $0.1$.
After the warm-up period, $\mu$ is gradually increased during training up to a maximum scaling factor of $2.0$.
This schedule allows the predictor to first learn a reasonable fit before strongly enforcing the spectral resonance constraints.

\paragraph{Model selection.}
Unless otherwise specified, we select the best checkpoint based on validation performance and report the mean and standard deviation across multiple random seeds.

\begin{table}[htb!]
\centering
\caption{Summary of experimental settings used across all tasks.}
\label{tab:exp_settings}
\begin{tabular}{ll}
\toprule
\textbf{Component} & \textbf{Configuration} \\
\midrule
Double Pendulum samples & 32{,}000 \\
Top Tagging samples & 64{,}000 \\
Epochs & 40 \\
Optimizer & Adam \\
Learning rate & $2\times 10^{-3}$ \\
Predictor $\phi_w$ & MLP (3 hidden layers, ReLU) \\
Warm-up epochs & 10 \\
Initial resonance weight $\mu$ & 0.1 \\
Maximum $\mu$ scaling & 2.0 \\
Model selection & Best validation checkpoint \\
\bottomrule
\end{tabular}
\end{table}

\subsection{Evaluation Metrics}
\label{sup_sec:metrics}

We report the following metrics:

\paragraph{(M1) Test MSE.}
Mean squared error evaluated on the held-out test set.

\paragraph{(M2) Invariance Error.}
When the ground-truth generator $X_{\mathrm{gt}}$ is available, we measure invariance by sampling transformations 
$R(t)=\exp(tX_{\mathrm{gt}})$ and computing
\[
\mathrm{InvErr}
=
\mathbb{E}_{x,t}
\big[
\|f_\theta(x)-f_\theta(R(t)x)\|_2^2
\big].
\]

\paragraph{(M3) Generator Recovery.}
We measure generator recovery using cosine similarity between the learned generator matrix $X_\theta$ and the ground-truth generator $X_{\mathrm{gt}}$:
\[
\mathrm{CosSim}(X_\theta,X_{\mathrm{gt}})
=
\frac{
\langle \mathrm{vec}(X_\theta), \mathrm{vec}(X_{\mathrm{gt}})\rangle
}{
\|\mathrm{vec}(X_\theta)\|_2
\|\mathrm{vec}(X_{\mathrm{gt}})\|_2 + \epsilon
}.
\]

\section{Sensitivity Analysis}
\label{sup_sec:sensitivity_analysis}

\begin{figure*}[t]
\centering
\includegraphics[width=\textwidth]{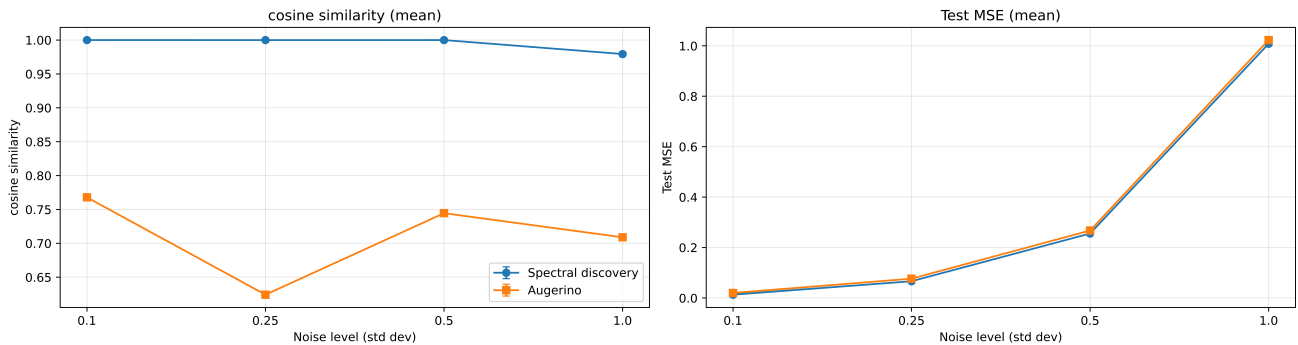}
\caption{
\textbf{Robustness to noise (noise sweep).}
Performance comparison between Spectral Discovery and Augerino on the 6D double pendulum task across increasing additive noise levels.
Left: Mean cosine similarity between the learned and ground-truth symmetry generators.
Right: Mean test loss.
Spectral Discovery maintains near-perfect generator recovery under low and moderate noise and degrades gracefully at high noise, while Augerino exhibits lower alignment and higher variability.
}
\label{fig:noise_sweep}
\end{figure*}

\begin{figure*}[t]
\centering
\includegraphics[width=\textwidth]{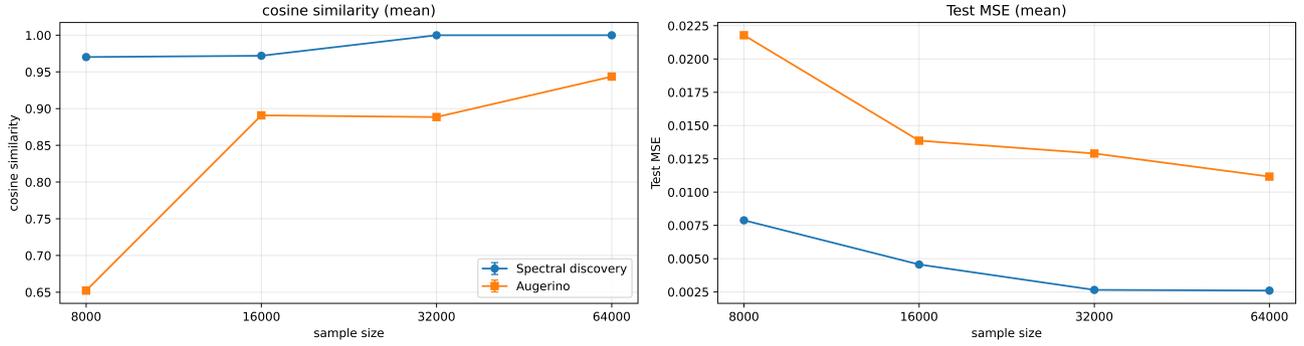}
\caption{
\textbf{Sample efficiency (sample sweep).}
Performance comparison between Spectral Discovery and Augerino as the number of training samples increases.
Left: Mean cosine similarity between the learned and ground-truth symmetry generators.
Right: Mean test mean-squared error.
Spectral Discovery achieves near-perfect symmetry recovery with fewer samples and consistently lower predictive error, while Augerino improves more gradually with data.
}
\label{fig:sample_sweep}
\end{figure*}

We compare Spectral Discovery and Augerino on the 6D double pendulum system, evaluating:
(i) recovery of the latent one-parameter symmetry generator via cosine similarity, and
(ii) predictive accuracy via test loss.

\paragraph{Noise robustness.}
Figure~\ref{fig:noise_sweep} presents performance as Gaussian noise increases from $0.1$ to $1.0$.
Spectral Discovery maintains near-perfect cosine similarity for low and moderate noise levels and exhibits only moderate degradation at the highest noise level.
In contrast, Augerino consistently recovers a less aligned generator across all noise levels.

Both methods experience increasing test loss as noise grows, as expected.
However, Spectral Discovery preserves substantially stronger symmetry alignment throughout the entire noise regime.
This indicates that frequency-based symmetry identification remains stable even when the signal-to-noise ratio decreases.

\paragraph{Sample efficiency.}
Figure~\ref{fig:sample_sweep} evaluates performance as the number of training samples increases from $8$k to $64$k.
Spectral Discovery achieves near-perfect generator recovery even at moderate sample sizes and maintains this alignment as more data are provided.
Augerino improves with additional samples but remains less accurate in identifying the underlying subgroup direction.

In terms of predictive performance, Spectral Discovery consistently attains lower test error across all sample sizes.
This suggests that explicitly exploiting spectral structure leads to more efficient use of available data.

\paragraph{Overall interpretation.}
Across both sweeps, Spectral Discovery demonstrates:
\begin{itemize}
    \item stronger recovery of the latent symmetry generator,
    \item improved robustness to noise, and
    \item better sample efficiency in predictive learning.
\end{itemize}

These findings support the claim that structured spectral analysis provides a principled mechanism for automatic discovery of continuous one-parameter symmetries, enabling reliable structural identification alongside strong predictive performance.

\subsection{Comparison with LieGAN}
\label{sup_sec:comp_lieGAN}


\begin{figure}[t]
\centering
\includegraphics[width=0.6\linewidth]{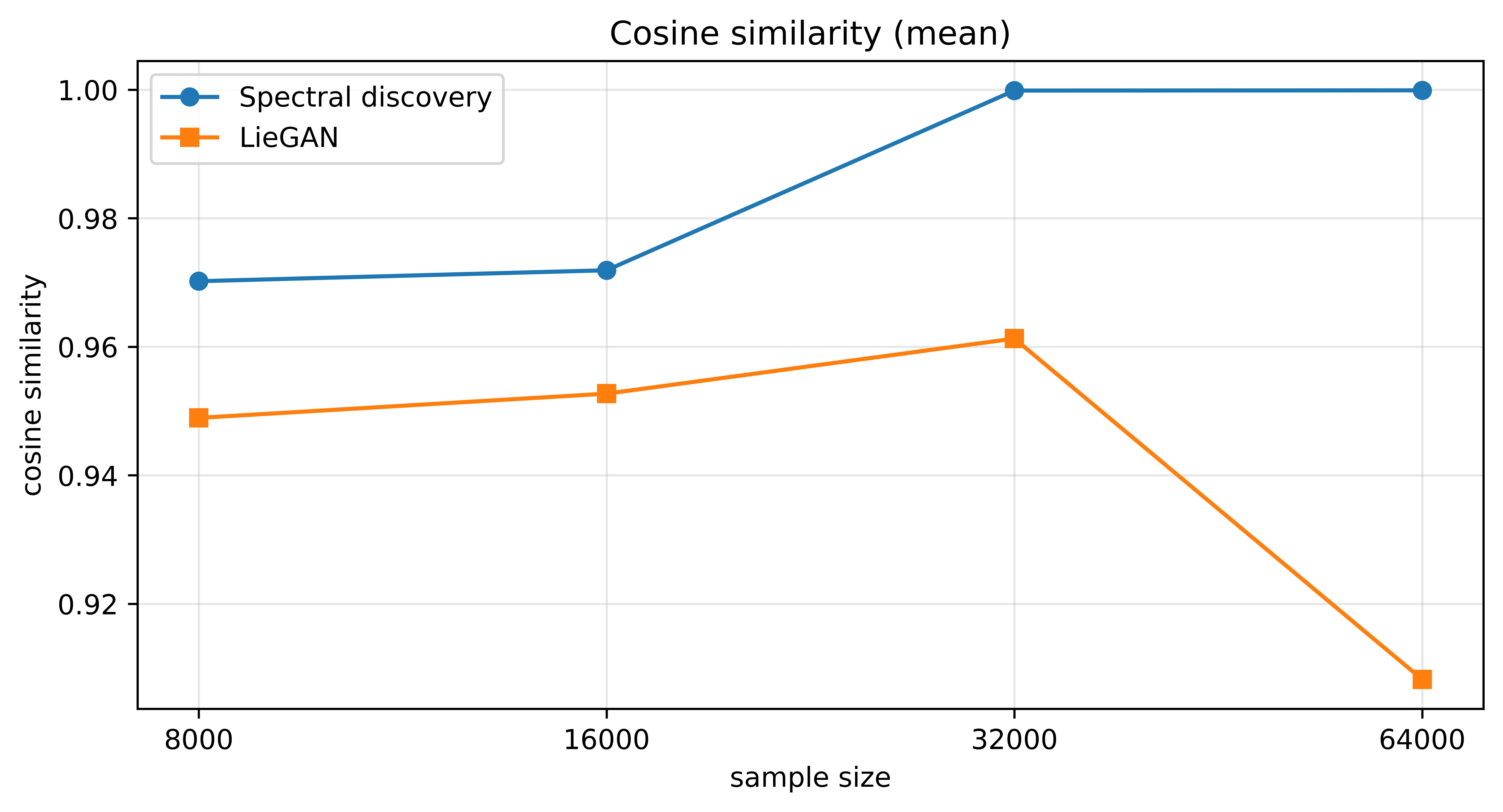}
\caption{
\textbf{LieGAN comparison: sample sweep (mean cosine similarity).}
Mean cosine similarity between the recovered and ground-truth symmetry generators on the 6D double pendulum task as the number of training samples increases.
Spectral Discovery achieves higher generator alignment across all sample sizes, reaching near-perfect recovery at larger sample regimes.
}
\label{fig:liegan_sample_sweep}
\end{figure}

\begin{figure}[t]
\centering
\includegraphics[width=0.6\linewidth]{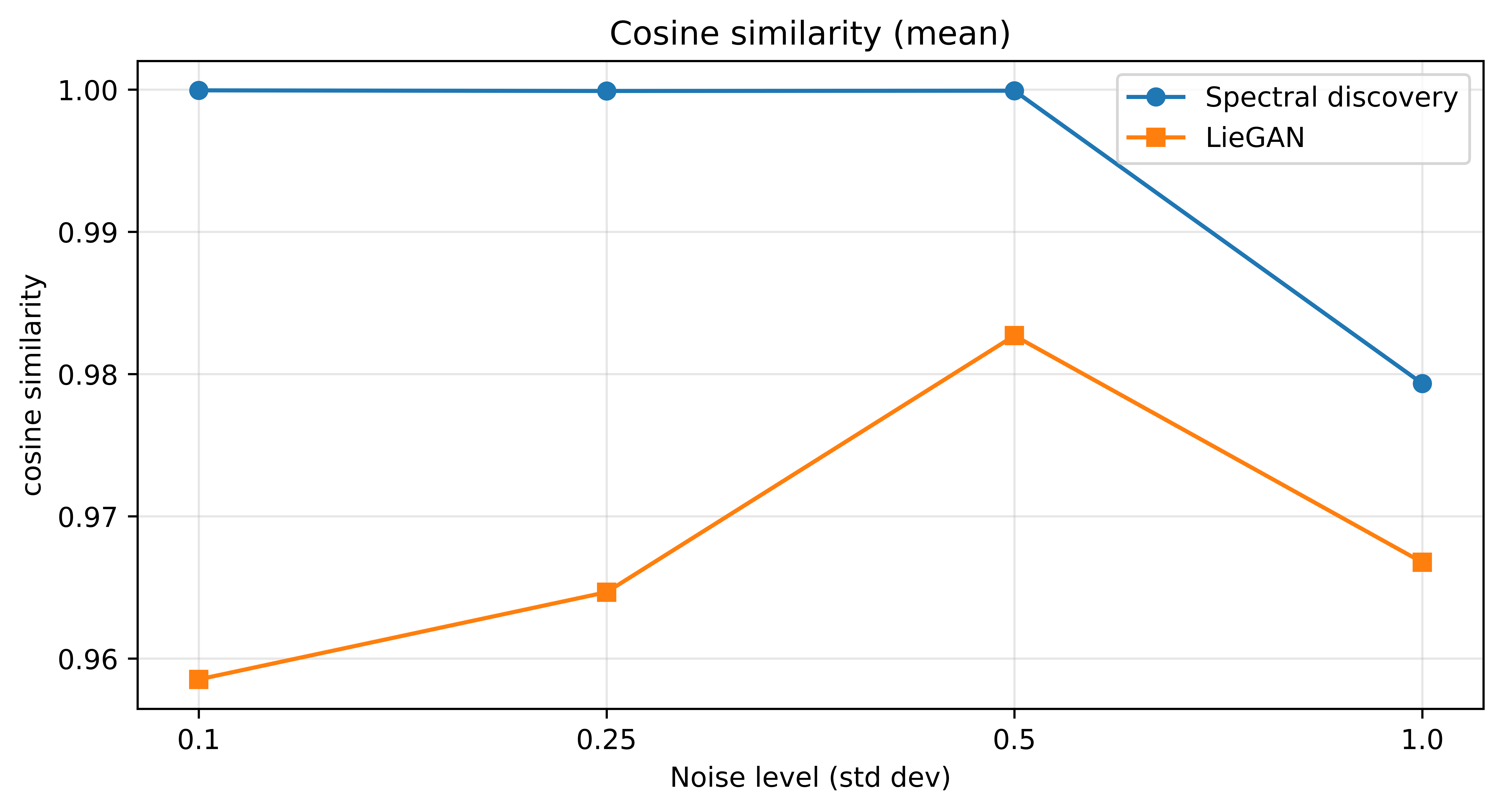}
\caption{
\textbf{LieGAN comparison: noise sweep (mean cosine similarity).}
Mean cosine similarity between the recovered and ground-truth symmetry generators on the 6D double pendulum task as additive noise increases.
Spectral Discovery maintains near-perfect alignment for low and moderate noise and degrades gracefully at high noise, while LieGAN exhibits lower alignment throughout.
}
\label{fig:liegan_noise_sweep}
\end{figure}

We compare Spectral Discovery with LieGAN on the 6D double pendulum benchmark, focusing exclusively on recovery of the latent one-parameter symmetry generator. Since LieGAN is designed purely for symmetry discovery and does not include a downstream predictive model, we report only mean cosine similarity between the recovered and ground-truth generators (please refer Table ~\ref{tab:comparison} for further details).

\paragraph{Sample efficiency.}
Figure~\ref{fig:liegan_sample_sweep} shows generator recovery as the number of training samples increases.
Spectral Discovery consistently achieves higher cosine similarity across all sample sizes and approaches near-perfect alignment in larger data regimes.
LieGAN improves with more samples but remains less aligned, particularly at higher sample counts.
This indicates that exploiting structured spectral concentration enables more accurate identification of the subgroup direction than adversarially learning generators from data.

\paragraph{Noise robustness.}
Figure~\ref{fig:liegan_noise_sweep} evaluates robustness under increasing additive noise.
Spectral Discovery maintains near-perfect alignment for low and moderate noise levels and degrades gracefully at high noise.
LieGAN also improves up to moderate noise but exhibits lower alignment overall and a more noticeable decline at the highest noise level.
These results suggest that frequency-based symmetry identification retains structural information more reliably under perturbations.

\paragraph{Overall interpretation.}
Across both sweeps, Spectral Discovery demonstrates stronger recovery of the latent symmetry generator than LieGAN.
Because LieGAN does not learn a predictive mapping, this comparison isolates the symmetry identification component.
The results support the central premise of our method: leveraging representation-theoretic structure and spectral concentration provides a principled and reliable mechanism for automatic discovery of continuous one-parameter symmetries.

\end{document}